\documentclass[letterpaper]{article} 
\usepackage{aaai2026}  
\usepackage{times}  
\usepackage{helvet}  
\usepackage{courier}  
\usepackage[hyphens]{url}  
\usepackage{graphicx} 
\usepackage{natbib}  
\usepackage{caption} 
\usepackage{algorithm}
\usepackage{algorithmic}
\usepackage{amsmath, amssymb, amsthm}
\newtheorem{theorem}{Theorem}

\usepackage{multirow}
\usepackage{booktabs}
\usepackage{xcolor}
\usepackage{makecell}
\usepackage{tabularx}
\usepackage{pifont}
\usepackage{bbding}

\newcommand{\dingcheck}{\ding{52}}   
\newcommand{\dingcross}{\ding{56}}     

\usepackage{newfloat}
\usepackage{listings}
\DeclareCaptionStyle{ruled}{labelfont=normalfont,labelsep=colon,strut=off} 
\floatstyle{ruled}
\newfloat{listing}{tb}{lst}{}
\floatname{listing}{Listing}

\title{Training-Free ANN-to-SNN Conversion for High-Performance \\ Spiking Transformers}
\author{
    Jingya~Wang\textsuperscript{\rm 1}\equalcontrib,
    Xin~Deng\textsuperscript{\rm 1}\equalcontrib,
    Wenjie~Wei\textsuperscript{\rm 1}, 
    Dehao~Zhang\textsuperscript{\rm 1},
    Shuai~Wang\textsuperscript{\rm 1},
    Qian~Sun\textsuperscript{\rm 1},
    Jieyuan~Zhang\textsuperscript{\rm 1},
    Hanwen~Liu\textsuperscript{\rm 1},
    Ning~Xie\textsuperscript{\rm 1},
    Malu~Zhang\textsuperscript{\rm 1,\rm 2 \equalcorr} 
}
\affiliations{
    \textsuperscript{\rm 1}University of Electronic Science and Technology of China,\\
    \textsuperscript{\rm 2}Shenzhen Loop Area Institute
    
    jiangyawang@std.uestc.edu.cn, dengbujiu@std.uestc.edu.cn, maluzhang@uestc.edu.cn
}

\usepackage{bibentry}

\begin{document}
\maketitle

\begin{abstract}
Leveraging the event-driven paradigm, Spiking Neural Networks (SNNs) offer a promising approach 
 for energy-efficient Transformer architectures.
While ANN-to-SNN conversion avoids the high training cost of directly trained Spiking Transformers, existing approaches still struggle to handle the nonlinear operations within Transformer blocks, and often require additional fine-tuning of pretrained ANNs.
To address these limitations, we propose a training-free and high-performance ANN-to-SNN conversion framework tailored for Transformer architectures. 
Specifically, we introduce a Multi-basis Exponential Decay (MBE) neuron that combines exponential decay with a multi-basis encoding strategy to effectively approximate nonlinear operations, eliminating the need for weight modifications in pretrained ANNs.
Extensive experiments across diverse tasks (CV, NLU, NLG) and mainstream Transformer architectures (ViT, RoBERTa, GPT-2) demonstrate that our method achieves near-lossless conversion accuracy with significantly lower latency. 
This provides a promising pathway for the efficient and scalable deployment of Spiking Transformers in real-world applications.
\end{abstract}

\section{Introduction}
Spiking Neural Networks (SNNs) have garnered attention due to their sparse and spike-driven computing paradigm~\cite{maass1997networks}. 
Unlike Artificial Neural Networks (ANNs), SNNs employ sparse binary spikes as information carriers, thereby offering superior energy efficiency for resource-limited devices~\cite{zhang2025toward, liang2025towards, wei2025qp}. Recently, several studies~\cite{wang2025spiking, xiao2025rethinking} have integrated high-performance Transformer architectures~\cite{vaswani2017attention} into SNNs to jointly exploit their expressive capacity and inherent energy efficiency. These methods substantially improve the performance of SNNs on complex tasks~\cite{wang2025ternary, wang2025snn, shan2025sdtrack, cai2025zooming}.

To obtain Transformer-based SNNs, two mainstream approaches exist: Direct Training (DT)~\cite{wu2018spatio} and ANN-to-SNN conversion (A2S) ~\cite{diehl2015fast}. 
DT employs surrogate gradients~\cite{neftci2019surrogate, wei2024q,sun2025temporal, zhang2025dendritic} and backpropagation through time (BPTT)~\cite{hochreiter1997long}, making the training of large-scale SNNs feasible. However, it suffers from inaccurate gradient approximation and \(\mathcal{T} \times\) training overhead. 
In contrast, A2S establishes an equivalence between the firing rate of Integrate-and-Fire neurons and ReLU functions~\cite{bu2023optimal}, enabling SNNs to inherit the pretrained weights of ANNs with nearly lossless conversion in CNNs. 
However, complex nonlinear operations in Transformers pose substantial challenges for converting them into SNN architectures.

To address this issue, prior research~\cite{wang2023masked} examines nonlinear components in Transformers. According to their computational forms, these nonlinear operations can be categorized into single-variable (e.g., GELU, Tanh) and multi-variable operations (e.g., variable–variable floating-point multiplications, LayerNorm). SpikeZIP-TF~\cite{you2024spikezip} replaces single-variable functions and fine-tunes pre-trained ANN weights, incurring additional training overhead. STA~\cite{jiang2024spatio} approximates all nonlinear functions using group operators, which increases inference latency. 
Thus, there is a need for a conversion method that balances training cost with inference efficiency.

Inspired by the temporal coding mechanisms observed in biological neurons, Few-spikes (FS) neurons~\cite{stockl2021optimized} are proposed. It leverages spike timing and patterns to represent neuronal activation states. This characteristic enables it to achieve near-lossless performance in small-scale CNN-based conversion~\cite{mao2025fsnap}. However, as model scale increases, FS-based conversion methods exhibit substantial performance gaps. Moreover, their single-input fitting design limits their ability to handle the multi-variable operations in Transformer architectures.

In this paper, we first theoretically and empirically analyze the key challenges associated with FS–based conversion, particularly in large-scale network architectures. To overcome these challenges, we propose a training-free A2S conversion framework tailored for Transformer architectures. Specifically, we introduce a novel Multi-basis Exponential Decay (MBE) neuron that effectively approximates diverse nonlinear operations, enabling our conversion framework to achieve both low inference latency and near-lossless accuracy. The main contributions are as follows:

\begin{itemize}
    \item 
    We systematically analyze the incompatibility between FS neurons and Transformer architectures, focusing on excessive dependence on initialization (EDI) and global suboptimality (GSO) problems, which hinder convergence and degrade conversion performance.

    \item
    We introduce a MBE neuron with exponential decay strategy and multi-basis encoding method. The decay strategy enables multi-resolution representations, while multi-basis encoding enhances the near-lossless approximation of diverse nonlinear operations.

    \item
    We propose an A2S framework based on MBE neurons that efficiently approximates various nonlinear operations in Transformer architectures, including variable-variable FP multiplications, GELU, Softmax, and LayerNorm. It achieves near-lossless conversion with fast inference and require no training on source ANNs.  

    \item
    Extensive experiments on diverse tasks (CV, NLU, NLG) and architectures (ViT, RoBERTa, GPT-2) demonstrate that our method achieves near-lossless conversion with reduced latency, yielding competitive results among existing Transformer-based A2S methods.
    
\end{itemize}

\section{Related Work}

Existing conversion paradigms can be categorized into two types based on whether extra training on the source ANNs is required: training-dependent and training-free.

\textit{\textbf{Training-dependent conversion}} typically requires additional training of ANN based on theoretical ANN-SNN equivalence. 
Clip-Floor-Shift activations~\cite{bu2023optimal} replace ReLU to better match spiking neuron behavior.
A two-phase training~\cite{ding2021optimal} first optimizes ANN weights, then adjusts neuron thresholds, while activation range constraints and membrane potential initialization~\cite{bu2022optimized} help further reduce conversion error.
\cite{wang2022towards} introduces a two-stage approach to handle quantization, pruning, and residual errors. 
 \cite{jiang2023unified} proposes a unified framework treating spike-rate mapping as a differentiable problem.
For Transformer architectures, a QANN is integrated with the spatiotemporal properties of spiking neurons for lossless conversion \cite{you2024spikezip}. 
While effective, these methods introduce computational overhead through additional training of source ANNs.

\textit{\textbf{Training-free conversion}} directly transforms pre-trained ANNs via structure and parameter reuse without fine-tuning. Early works align SNN thresholds with ANN activations via threshold balancing~\cite{diehl2015fast,rueckauer2017conversion}, but longer timesteps are required.
Signed spiking neurons~\cite{wang2022signed} support dynamic vision data.
In~\cite{jiang2024spatio}, training-free Transformer conversion is achieved using universal group operators and spatial rectification self-attention. 
SpikedAttention~\cite{hwang2024spikedattention} enables spike-driven Transformer A2S via trace-based matrix multiplication and winner-take-all spike shifting, yet retains non-spiking LayerNorm.
ECMT~\cite{huang2024towardsecmt} reduces latency in Transformer SNNs, yet still relies on floating-point operations.
While retraining is avoided, these methods often require long timesteps or retain non-spiking components, limiting SNNs' energy efficiency.

\section{Preliminary \& Problem Analysis}
\subsection{Preliminary}
Inspired by authentic electrophysiological characteristics of human brain neurons, FS neuron  is proposed to address the inherent trade-off between spike count and accuracy in A2S \cite{stockl2021optimized}. By introducing learnable parameters including neuron threshold, membrane potential reset value, and spike intensity, FS neuron approximates ANN activation functions with very few spikes. Mathematically, the membrane potential of FS neuron can be computed as:
\begin{equation}\label{eq1} 
u[t+1] = u[t] - r[t]\cdot s[t],
\end{equation}
where $t\in [0,T-1]$ denotes timesteps, $u[t]$ represents the membrane potential, $s[t]\in\{0,1\}$ is the binary spike, and $r[t]$ is the learned reset value.
The membrane potential has no decay, and its initial value is typically set to the input from the previous layer, i.e., $u[0]=x$. When the membrane potential exceeds a learned threshold, a spike is generated, which is described as follows:
\begin{equation}
s[t] = \mathcal{H}\left(\left(x - \sum\nolimits_{t^{'}=0}^{t-1} r[t^{'}]\cdot s[t^{'}]\right) - V_{th}[t]\right), 
\end{equation}
where $\mathcal{H} (\cdot) $ is the Heaviside step function and $V_{th}[t]$ is the learned firing threshold. 
After each spike emission, the reset mechanism is invoked to update the membrane potential, as described in Eq.(\ref{eq1}). 
Noteworthy, the spike $s[t]$ emitted by FS neuron at time $t$ is amplified by learned spike intensity $d[t]$. After $T$ timesteps, the neuron integrates the weighted spike train and forwards it to the next layer:
\begin{equation}\label{eq3}
\hat{f}(x) = \sum\nolimits_{t=0}^{T-1} d[t]s[t] \approx f{(x)},
\end{equation}
where $\hat{f}(x)$ is the approximation value produced by FS neuron and $f(x)$ is the output of activation function in ANNs.

\subsection{Problem Analysis}
\label{section:3.2}
We conduct a systematic analysis of FS neurons to understand their limitations in Transformer-based A2S tasks. 

\begin{figure}[h]
  \centering
  \includegraphics[width=0.95\linewidth]{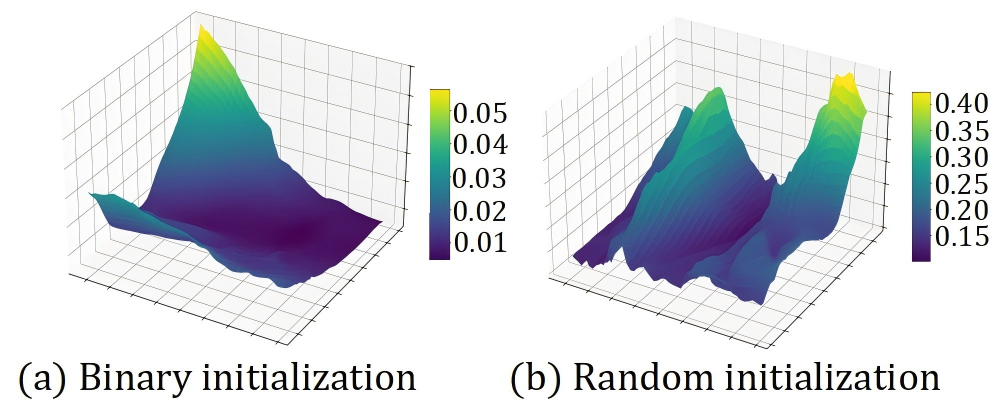}
  \caption{Loss landscape of FS neurons in approximating the nonlinear activation function(i.e., ReLU). The 3D landscape of $\mathcal{L}_{Binary}$ and $\mathcal{L}_{Random}$ from two different initialization. }
  \label{fig:losslandspace}
  \vspace{-10pt}
\end{figure}

\paragraph{\textbf{Approximation Error Analysis}} 
To qualitatively analyze how the FS neuron design impacts performance, following~\cite{jiang2024spatio}, we consider errors from three sources: insufficient sampling, limited parameterization, and spiking quantization. This yields the following error bound:

\begin{theorem}[FS Neuron Error Bound]
Let \( f: [a, b] \rightarrow \mathbb{R} \) be a target activation function, and let \( \hat{f}_T^{(M)}(x) \) denote the output of an FS neuron with \( T \) timesteps trained on \( M \) samples. Then the total approximation error satisfies: 
\[
\varepsilon\leq \mathcal{O}\Biggl( 
\underbrace{ \sqrt{ \frac{ T \log T \log M }{ M } } }_{\text{Empirical Gap}} +
\underbrace{ \frac{ \mathcal{L}_f |y|_{\max} }{ T } }_{\text{Parametric Gap}} +\!
\underbrace{ \frac{ \| d \|_1 }{ T } }_{\text{Quantization Gap}} 
\Biggr),
\]
where \( \varepsilon = \mathbb{E}[|f(x) - \hat{f}_T^{(M)}(x)|] \),  
\( \mathcal{L}_f \) is the Lipschitz constant of \( f \),  
\( |y|_{\max} \) the maximum value of \( f(x) \), and  
\( \| d \|_1 \) is the \( L_1 \)-norm of spike intensities. Proof in Appendix~A. 
\label{theorem 1}
\end{theorem}

\paragraph{\textbf{Excessive Dependence on Initialization}} 

From Theorem~\ref{theorem 1}, the quantization gap $\frac{\|d\|_1}{T}$ depends on spike intensities $d[t]$, which are determined through optimization from initialization. Poor initialization may produce suboptimal $d[t]$ that increase $\|d\|_1$, thereby enlarging the approximation error. We conduct experiments to validate this: random initialization yields $\text{MSE}_\text{random} = 1.5 \times 10^{-3}$, while binary initialization ($V_{\mathrm{th}}[t] = r[t] = d[t] = 2^{T - t}$) achieves $\text{MSE}_\text{binary} = 9.4 \times 10^{-5}$, representing a 93.29\% performance degradation under random initialization. As shown in Fig. \ref{fig:losslandspace}, binary initialization  results in a smoother and more stable optimization landscape, while random initialization leads to multiple local minima. This further corroborates the sensitivity of FS neurons to initialization quality.

 \begin{figure}[h]
    \centering
    \includegraphics[width=\linewidth]{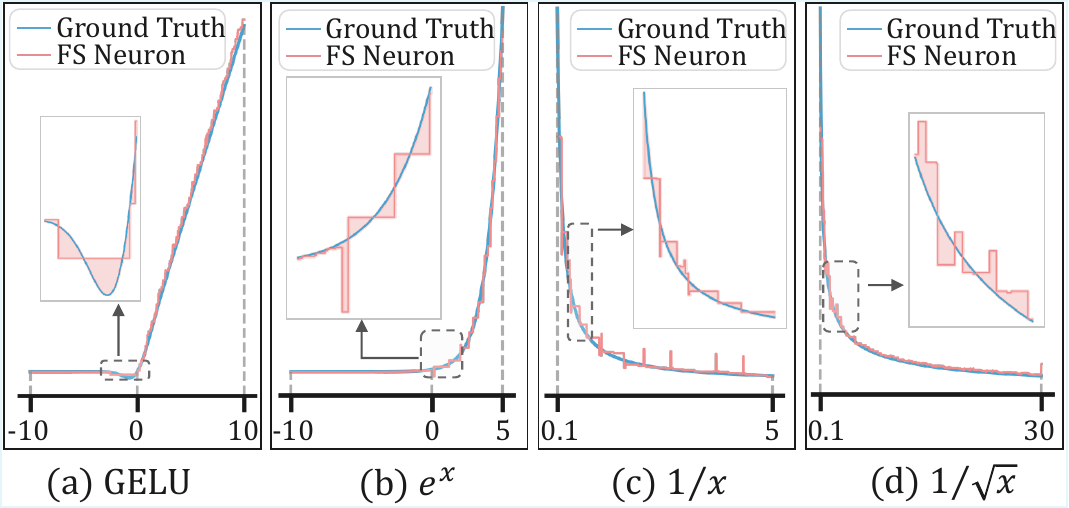}
    \caption{Critical interval approximation failure of nonlinear operations in Transformer using FS neurons.}
    \label{fig:fitting error in Transformer}
\end{figure}

\paragraph{\textbf{Global Suboptimality}} 

GSO stems from the parametric gap term \( \frac{\mathcal{L}_f |y|_{\max}}{T} \) in Theorem~\ref{theorem 1}, revealing FS neurons' limitations when approximating functions with non-uniform complexity.
Transformer nonlinearity inputs concentrate near zero (Statistical analysis in Appendix~C), where activation functions such as GELU and SiLU have high curvature and large local Lipschitz constants \( \mathcal{L}_f^{\text{local}} \gg \mathcal{L}_f \). This creates an amplified local gap \( \varepsilon_{\text{param}}^{\text{local}} = \mathcal{O}(\mathcal{L}_f^{\text{local}} |y|_{\max}^{local} / T) \), 
which dominates the approximation error and highlights the inadequacy of uniform time allocation.
Experimental results support this analysis: as shown in Fig.~\ref{fig:fitting error in Transformer}, although the overall approximation error is small, the fitting degrades significantly in the near-zero region where outputs vary most rapidly. This region is crucial for Transformer performance, and the poor local approximation dominates overall behavior (FS-based Transformer conversion results in Appendix~D). 

\section{Method}
In this section, we first propose the MBE neuron, which employs exponential decay parameter update strategy and multi-basis encoding method to map activation values at multiple resolutions. Based on the MBE neuron, we further design an A2S conversion framework to overcome key challenges for Transformer  architectures, achieving near-lossless conversion that is training-free and low-latency.

\subsection{Multi-Basis Exponential Decay Neuron}
To address the limitations of FS neurons, we propose a novel MBE neuron with an exponential decay parameter update strategy and multi-basis encoding scheme. 
Unlike FS neurons requiring learning multiple parameters $(d[t], r[t], V_{\mathrm{th}}[t])$ per timestep, MBE neurons only learn decay rates and discrete timesteps, reducing parameter overhead and gradient instability (Details in Appendix~E).

\begin{figure}[h]
  \centering
  \includegraphics[width=1\linewidth]{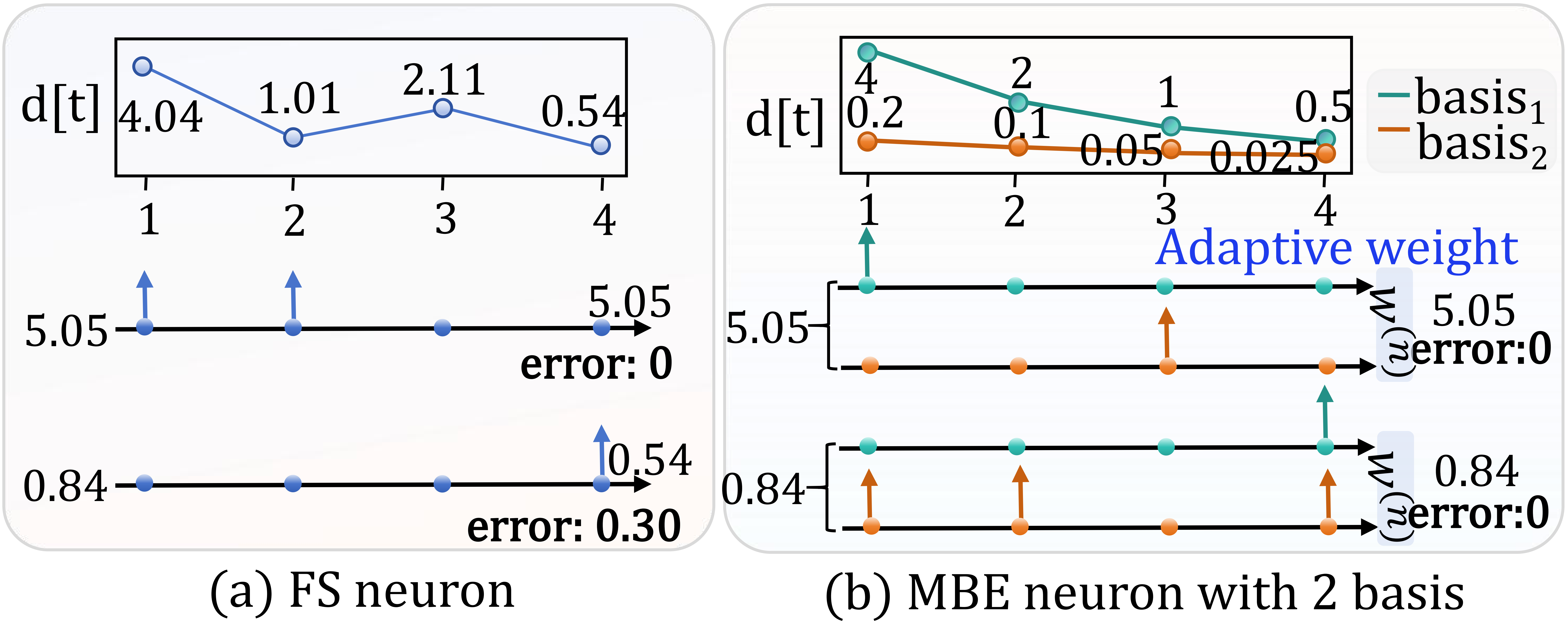} 
  \caption{FS neuron and MBE neuron encoding methods.}
  \label{fig:encoding}
\end{figure}

The exponential decay mechanism enables progressive refinement from coarse to fine resolution, allowing adaptive approximation of functions with varying granularity and rapid focus on limited value ranges over shorter timesteps, as illustrated in Fig.~\ref{fig:encoding}. The parameter update is defined as:
\begin{align}
Para(\tau_n, t) = \alpha \cdot exp\left(-\frac{t \Delta t}{\tau_n}\right), 
\end{align}
where $Para(\tau_n, t)$ represents the parameter value at time $t$, $\alpha$ is a hyperparameter typically set to the target function's maximum value, $\Delta t$ is the discrete timestep, and $\tau_n \in \{\tau_{d_n}, \tau_{r_n}, \tau_{V_{{th}_{n}}}\}$ are the corresponding decay rates.

\begin{figure}[t]
  \centering
  \includegraphics[width=\linewidth]{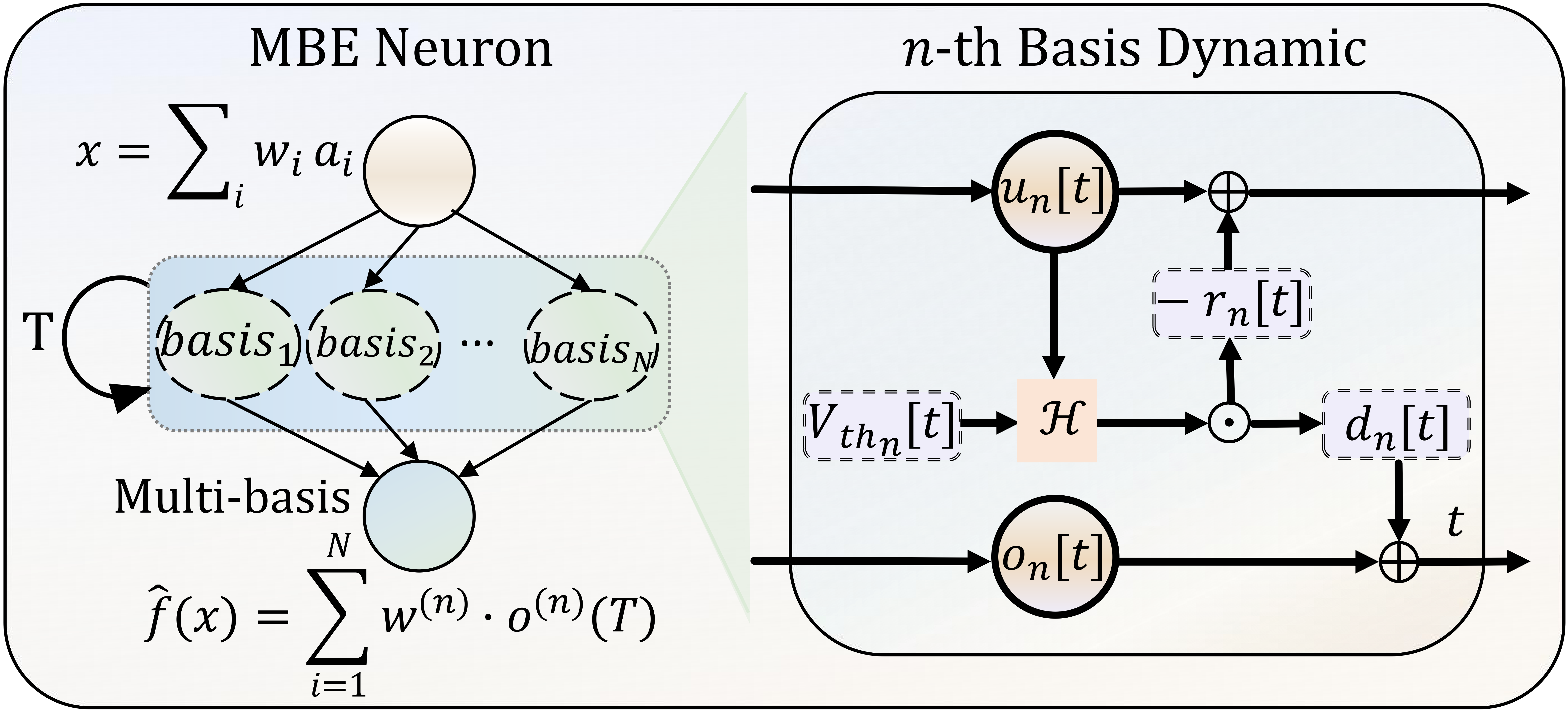}
  \caption{Multi-basis exponential decay neuron.}
  \label{fig:MBEN}
\end{figure}

To further enhance the representational capacity for approximating diverse nonlinear functions, we introduce multi-basis encoding as the second core feature. 
As shown in Fig.\ref{fig:MBEN}, each MBE neuron comprises $n$ basis components contributing distinct functional components to the overall response. The dynamics of the $n$-th basis are defined as:
\begin{align}
    u_n[t+1] &= u_n[t] - s_n[t] \cdot r_n[t],\\
    s_n[t] &=  \mathcal{H}\big(u_n[t] - {V_{th}}_n[t]\big),\\
    o_n[t+1] &= o_n[t] + s_n[t] \cdot d_n[t], 
\end{align}
where $u_n[t]$ is the membrane potential of basis $n$ at time $t$, $o_n[t]$ is the accumulated output, $d_n[t]$ denotes the spike intensity, $r_n[t]$ is the reset value, ${V_{th}}_n[t]$ is the firing threshold, and $\mathcal{H}(\cdot)$ is the Heaviside step function. When $u_n[t] \geq {V_{th}}_n[t]$, the neuron fires and emits a spike with intensity $d_n[t]$ while reducing the membrane potential by $r_n[t]$. 

Finally, all basis outputs are weighted by a \( w \) to form \( \hat{f}(x) \), passed to the next layer to approximate \( f(x) \):
\begin{equation}
\label{appro}
\hat{f}(x) = \sum\nolimits_{n=1}^{N} w^{(n)} \cdot \text{o}^{(n)}(T) \approx f(x).
\end{equation}

\paragraph{\textbf{Approximation Error Analysis}} 
       Following the same analytical framework in Theorem~\ref{theorem 1}, we derive error bounds for MBE neurons, accounting for the representational capacity:

\begin{theorem}[MBE Neuron Error Bounds]
For a target activation function $f: [a,b] \rightarrow \mathbb{R}$ and MBE neuron output $\hat{f}_{N,T}^{(M)}(x)$ with $N$ basis components, $T$ timesteps trained on $M$ samples, the total approximation error satisfies:
\begin{equation*}
\varepsilon^* \!\leq \!O\Biggl(
    \underbrace{
        \sqrt{
            \frac{ N \log (N) \log M }{ M }
        }
    }_{\text{Empirical Gap}} 
   \! +\!\!
    \underbrace{
        \frac{ \mathcal{L}_f |y|_{max} }{ NT }
    }_{\text{Parametric Gap}}\!\! +\!
    \underbrace{
         \frac{ \| w \|_1 |\alpha| \tau_{max} }{ T \Delta t }
    }_{\text{Quantization Gap}}
\Biggr),
\label{theorem 2}
\end{equation*}
where $\varepsilon^* = \mathbb{E}[|f(x) - \hat{f}_{N,T}^{(M)}(x)|]$, $\mathcal{L}_f$ is the Lipschitz constant of $f$, $|y|_{max}$ is the maximum absolute value of the target function, $\| w \|_1$ is the $L_1$ norm of basis weights, $|\alpha|$ is the positive scaling parameter, $\tau_{max} = \max_n \{\tau_{d_n}, \tau_{r_n}, \tau_{V_{th_n}}\}$ is the maximum time constant, and $\Delta t$ is the discrete timestep. The proof is provided in Appendix B. 
\end{theorem}

The theoretical analysis shows MBE advantages and guides our conversion framework through three key insights:

\textbf{Basis component design:} The Parametric Gap $\mathcal{O}(1/NT)$ outperforms FS neurons' $\mathcal{O}(\mathcal{L}_f |y|_{\max}/T)$ by enabling multi-basis encoding to adaptively focus on high-curvature regions, solving the GSO problem without extra timesteps.

\textbf{Initialization stability:} The Quantization Gap $\frac{ \| w \|_1 |\alpha| \tau_{max} }{ T \Delta t }$ mitigates the EDI problem through controlled scaling via $|\alpha|$ and time constant effects via $\tau_{max}/\Delta t$, enabling initialization-stability parameter updates.

\textbf{Hyperparameter determination:} The $1/NT$ scaling reveals a principled trade-off: increasing basis components $N$ can effectively  compensate for limited timesteps $T$, thereby achieving a balance between accuracy and low latency.

\begin{figure*}[t]
  \centering
  \includegraphics[width=0.98\linewidth]{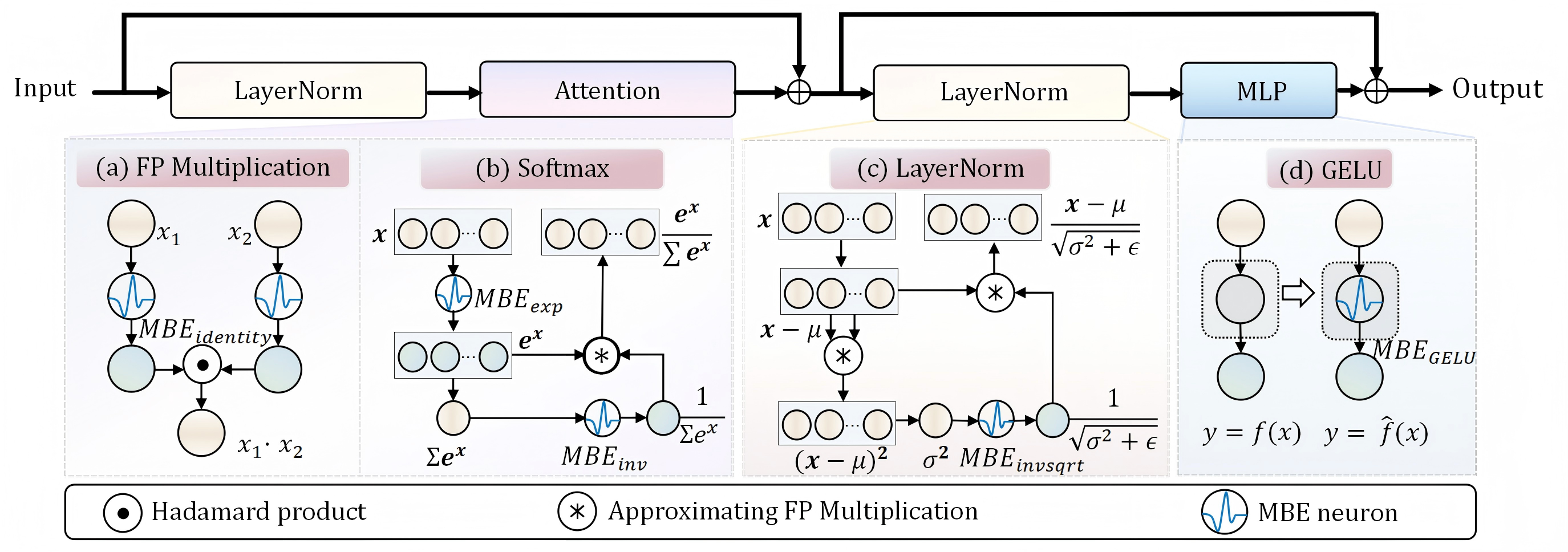}
  \caption{Overview of our framework, which depicts the approximations of FP multiplication, Softmax, LayerNorm, and GELU.}
  \label{fig:FFSC Framework}
\end{figure*}
\subsection{Conversion Framework}

Based on MBE neurons, we propose a framework for Transformer-to-SNN conversion. As shown in Fig.~\ref{fig:FFSC Framework}, by decomposing non-spiking components  (nonlinear activations, FP multiplication, Softmax, LayerNorm) 
into basic functions and designing corresponding MBE neurons for equivalent approximation, we enable spiking-based representation.

\paragraph{\textbf{Approximation of Nonlinear Activation Functions}} 
To enable accurate and efficient spike-based approximation of nonlinear activation functions such as GELU and Tanh in Transformers, we employ MBE neurons with \( N = 4 \) basis components (in Eq.(\ref{appro})). Since LayerNorm compresses activations into a narrow range, we constrain the GELU input domain to \( (-120, 10) \) to enhance both approximation accuracy and training stability. Within this interval, we uniformly sample \( M = 10{,}000 \) points from the target function \( f(x) \) and compute the spike-based output \( \hat{f}(x) \) over \( T \) timesteps.

\paragraph{\textbf{Approximation of Floating-Point Multiplication}} 
\label{spiking FP_MUL}
The FP multiplication operands $x_1$ and $x_2$ are transformed through approximate identity mapping by the MBE neuron $MBE_{Id}$, yielding the spike-train representations:
\begin{equation}
    x \approx MBE_{Id}(x) = \sum_{t=0}^{T-1}d[t]s[t], \quad x \in \{x_1, x_2\}.
\end{equation}
Then, the multiplication of $x_1$ and $x_2$ is expressed as:
\begin{align}
\label{eq:mul1}
    x_1 \cdot x_2 &\approx\sum_{i=0}^{T-1} \sum_{j=0}^{T-1} d[i] d'[j] \cdot s[i] s'[j].
\end{align}
We represent the spike sequences of two MBE neurons as vectors $\mathbf{s}=\{s[t]\}_{t=0}^{T-1}$ and $\mathbf{s'}=\{s'[t]\}_{t=0}^{T-1}$ respectively, and denote their corresponding intensity sequences as vectors $\mathbf{d}=\{d[t]\}_{t=0}^{T-1}$ and $\mathbf{d'}=\{d'[t]\}_{t=0}^{T-1}$. Based on these temporal vectors, we construct the intensity matrix  $\mathbf{D} \in \mathbb{R}^{T \times T}$ and the spike matrix $\mathbf{S} \in \{0,1\}^{T \times T}$ as follows:
\begin{align}
\mathbf{D} &=  \mathbf{d}^\top \mathbf{d'}, \qquad \mathbf{S} = \mathbf{s}^\top \mathbf{s'}.
\end{align}
To achieve multiplication in the form of spikes, we employ the Hadamard product followed by summation. Consequently, Eq.(\ref{eq:mul1}) can be expressed as:
\begin{align}
    x_1 \cdot x_2 \approx \sum_{i=0}^{T-1}\sum_{j=0}^{T-1}\mathbf{D}_{ij}\cdot \mathbf{S}_{ij}
    =\sum_{i=0}^{T-1}\sum_{j=0}^{T-1}(\mathbf{D}\odot \mathbf{S})_{ij},
\end{align}
 where  $\mathbf{D}_{ij}$ = $d[i]d'[j]$ and $\mathbf{S}_{ij}$ = $s[i]s'[j]$. The intensity matrix $\mathbf{D}$ can be precomputed and shared across the entire network. With the binary matrix $\mathbf{S}$, the entire FP multiplication approximation process maintains its spike-driven characteristics (Details in Appendix.~F.1).
By replacing FP multiplications in self-attention with our spike-driven operations, we enable fully spike-based attention matrix computation.

\paragraph{\textbf{Approximation of Softmax}} 
The $\text{Softmax}(x_i) = \frac{e^{x_i}}{\sum_j e^{x_j}}$, where $i$ ,$j$ denote the indices of elements in the sequence. 
Since softmax depends on all inputs, it cannot be directly implemented by a single MBE neuron. As shown in Fig.~\ref{fig:FFSC Framework}(b), we decompose the softmax computation into three components: exponential ($e^x$), reciprocal ($1/x$), and FP multiplication, each approximated by MBE neurons.
For the term $e^{x_{i}}$,  we apply the change-of-base formula to decompose it into integer and decimal components as:
\begin{equation}
\label{equation:MBE_exp}
    e^{x_i} = 2^{x_i\cdot\log_2e}=2^{\lfloor {x_i}\cdot\log_2e\rfloor} \cdot 2^{x_i\cdot\log_2e - \lfloor {x_i}\cdot\log_2e\rfloor},
\end{equation}
where $\lfloor{\cdot}\rfloor$ denotes the floor operation. 
The fractional part $2^{x_i\cdot\log_2e - \lfloor {x_i}\cdot\log_2e\rfloor}$ is approximated by MBE neurons, and the integer exponent $2^{\lfloor x_i \cdot \log_2 e \rfloor}$ is achieved by hardware-efficient addition \cite{li2023vit}.
For the reciprocal term \( 1 / \sum_j e^{x_j} \), we follow the IEEE 754 standard~\cite{kahan1996ieee} to extract the exponent \( E \) and mantissa \( M \) through the corresponding bits of $\sum_j e^{x_j}$ (i.e., $\sum_j e^{x_j}=M\cdot2^E$). The reciprocal \( 1/M \) is directly approximated using an MBE neuron, and the final result is obtained as \( 1 / \sum_j e^{x_j} = 2^{-E} / M \).

Combining the approximations of \( e^{x_i} \) and \( 1 / \sum_j e^{x_j} \), together with our FP multiplication approximation scheme, we compute the final softmax output in a spike-based manner (Detailed Softmax algorithm table in Appendix~F.2).

\paragraph{\textbf{Approximation of LayerNorm}} The 
$\text{LN}(x_i)=\gamma\cdot\frac{x_i-\mu}{\sqrt{\sum(x_i-\mu)^2/n+\epsilon}}+ \beta$,
where $\mu$ is the input mean, $n$ is the number of elements, $\epsilon$ is a stability constant, and $\gamma, \beta$ are learnable parameters.
As shown in Fig.~\ref{fig:FFSC Framework}(c), we also decompose LayerNorm into three spike-unfriendly operations: squaring ($x^2$), reciprocal square root ($1/\sqrt{x}$), and FP multiplication, each approximated by MBE neurons.

For the squared term $\sum (x_i-\mu)^2$, we adopt the proposed FP multiplication approximation with the intensity matrix $\mathbf{D}$ pre-scaled by $1/n$ to directly yield $(x_i - \mu)^2 / n$.

For the inverse square root $1/\sqrt{\sum (x_i-\mu)^2/n+\epsilon}$, we decompose the variance term into exponent $E$ and mantissa $M$. We adjust $E$ and $M$ based on $E$'s parity to ensure $-E/2$ is an integer. The term $1/\sqrt{M}$ is approximated directly using MBE neuron, and the inverse square root is obtained by multiplying with $2^{-E/2}$, similar to the Softmax approximation.
Finally, we multiply $(x_i-\mu)$ with the approximated $1/\sqrt{\sum(x_i-\mu)^2/n+\epsilon}$ using our FP multiplication approximation to obtain the LayerNorm output.

Based on the aforementioned conversion of all non-spike-friendly components into equivalent spiking forms, we replace the corresponding modules to construct the Transformer-based SNN without modifying the original network parameters. Detailed pseudocode is in Appendix F.3.

\begin{table*}[htbp]
\setlength{\tabcolsep}{4.0pt}
\centering
\begin{tabularx}{\textwidth}{llccccccc}
\toprule
\multicolumn{1}{c}{Category} & Method & Arch. & SD & TF & Param[M] & Timesteps & ANN &Acc.[\%]($\Delta$) \\
\midrule
\multirow{8}{*}{DT} 

&DSR\cite{meng2022training}& ResNet-18 & {\dingcheck}& - & 12 & 50 & - & 67.74\\
& TET\cite{deng2022temporal} & SEW-ResNet-34 &{\dingcross} & - & 22 & 4 & - &68.00\\
& AT-SNN\cite{yao2023attention} & ResNet-104 & {\dingcross} & - & 45 & 4 & - & 77.08 \\
\cmidrule{2-9} 
& Spikingformer\cite{zhou2023spikingformer} & -4-384-400E &{\dingcheck} & - & 66 & 4 & - &75.85 \\
& SDTv1\cite{yao2023spikedriven} & -8-768 & {\dingcheck} & - & 66 & 4 & - &77.07 \\
& Spikeformer\cite{li2022spikeformer} & -7L/3×2×4 &{\dingcross} & - & 38 & 4 & - &78.31 \\
& QKFormer\cite{zhou2024qkformer} & HST-10-768 & {\dingcheck} & - & 65 & 4 & - & 84.22 \\
& SDTv3\cite{yao2025scaling} & E-SpikeFormer & {\dingcheck} & - & 173 & 8 & - & 85.10\\
\cmidrule{1-9} 
\multirow{14}{*}{A2S} 

& QCFS \cite{bu2023optimal} & VGG-16 & {\dingcheck} & {\dingcross} & 138 & 64 & 74.92 &72.85(-2.07)\\
&QFFS\cite{li2022quantization} &VGG-16 &{\dingcheck} & {\dingcross} & 138 & 8 &73.08 &73.10(+0.02)\\
& SNM\cite{wang2022signed} & ResNet-18 & {\dingcheck} & {\dingcheck} & 12 & 64 & 73.18 & 71.50(-1.68)\\
& QCFS\cite{bu2023optimal} & ResNet-34 & {\dingcheck} & {\dingcross} & 22 & 64 & 74.23 & 72.35(-1.88)\\
& ECL\cite{liu2025efficient} & ResNet-34 & {\dingcheck} & {\dingcross} & 22 & 16 & 74.36 & 72.37(-1.99)\\
& AdaFire\cite{wang2025adaptive} & ResNet-34 & {\dingcheck} & {\dingcheck} & 22 & 8 & 75.66 & 72.96(-2.70)\\
& \multirow{2}{*}{\textbf{Ours}} & VGG16 & {\dingcheck} & {\dingcheck} & 138 & 10 & 73.37 & 72.61 {(-0.76)} \\
& & ResNet-34 & {\dingcheck} &{\dingcheck} & 22 & 10 & 76.31 & \textbf{75.57(-0.74)} \\
\cmidrule{2-9}
& ECMT\cite{huang2024towardsecmt} & ViT-L/16 & {\dingcross} &{\dingcheck} & 307 & 12 & 84.88&84.71(-0.17) \\
& STA\cite{jiang2024spatio} & ViT-B/32 & {\dingcross} &{\dingcheck} & 88 & 256 & 83.60&82.79(-0.81) \\
& SpikeZIP-TF\cite{you2024spikezip} & SViT-L-32Level & {\dingcheck} & {\dingcross} & 304 & 64 &85.41& 83.82(-1.59) \\
& SpikedAttention\cite{hwang2024spikedattention} & Swin-T(BN) & {\dingcross} & {\dingcheck} & 28 & 48 &79.30& 77.20(-2.10) \\ 

& \multirow{2}{*}{\textbf{Ours}} & ViT-B/16 & {\dingcheck} & {\dingcheck} & 86 & 16 & 83.44&83.00{(-0.44)} \\

& & ViT-M/16 & {\dingcheck} &{\dingcheck} & 64 & 16 &85.95& \textbf{85.31(-0.64)} \\
\bottomrule
\end{tabularx}
\caption{ImageNet performance comparison. Here, ``SD'' and ``TF'' represent Spike-Driven and Training-Free.}
\label{tab:imagenet_comprarison}
\end{table*}

\section{Experiments}

\subsection{Experimental Setup}
We conduct experiments on a GPU (RTX 4090) environment using the PyTorch framework, evaluating both pre-trained Vision Transformer (ViT) models, including ViT-Base-Patch16 (ViT-B/16) \cite{dosovitskiy2020image} and ViT-Medium-Patch16-Reg4-Gap-256 (ViT-M/16) \cite{darcet2023vision}, as well as CNN models, including VGG16 \cite{simonyan2014very} and ResNet34 \cite{he2016deep} on the ImageNet dataset \cite{deng2009imagenet}. To validate the generalizability of our method across different language domains and tasks, we adapt RoBERTa \cite{liu2019roberta, zhang2025spike} for natural language understanding (NLU) tasks and employ GPT-2 \cite{radford2019language} for natural language generation (NLG) task evaluation (Experimental and implementation details in Appendix~G.1).

\subsection{Comparative Study}

\paragraph{\textbf{Comparison on CV}} 
We evaluate the performance of our conversion framework by applying it to convert both ViT and CNN architectures. 
As shown in Tab.~\ref{tab:imagenet_comprarison}, ViT-B/16 and ViT-M/16 achieve conversion losses of 0.44\% and 0.64\% respectively—significantly lower than most existing methods. The framework attains 85.31\% accuracy for ViT and 75.57\% for CNN, outperforming other  SNNs. While ECMT achieves competitive accuracy at shorter timesteps, it retains FP multiplication in its expectation compensation module. In contrast, our A2S framework maintains crucial spike-based property, thereby avoiding the high energy consumption typically associated FP multiplication during inference. Furthermore, the training-free nature of our approach enables conversion with minimal computational overhead, eliminating the need for extensive retraining for source ANNs.
Our method successfully accomplishes A2S conversion without requiring additional training of the original network, preserves the essential spiking properties, delivers superior performance at short timesteps, and achieves optimal results across all evaluation dimensions.

\paragraph{\textbf{Comparison on NLU}}
To evaluate our method on NLU tasks, we conduct experiments on four benchmark datasets including SST-2, SST-5, MR, and Subj, following standard experimental settings from studies~\cite{you2024spikezip, zhu2023spikegpt}. As shown in Table~\ref{tab:nlu}, our method achieves competitive performance across all datasets, outperforming existing methods on both RoBERTa-Base (125M) and RoBERTa-Large (355M). On SST-2, it attains 95.98\% accuracy with $T=16$, representing only 0.24\% degradation from the original ANN (96.22\%), while outperforming SpikeZIP-TF~\cite{you2024spikezip} by 2.19\% and reducing timesteps by 87.5\%. Similarly, on MR with RoBERTa-Base, our method achieves 89.00\% accuracy at $T=16$, exceeding SpikeZIP-TF by 2.87\% with 75\% fewer timesteps.

\begin{table}[h]
  \setlength{\tabcolsep}{1.1pt} 
  \begin{tabular}{llcccccc}
    \toprule
    Category & Model & Param & SST-2 & SST-5 & MR & Subj & T \\
    \midrule
    \multirow{2}{*}{ANN} 
        & \multirow{2}{*}{\centering Roberta} 
            & 125 & 94.49 & 55.46 & 89.39 & 96.45 & 1 \\
        & & 355 & 96.22 & 59.37 & 91.36 & 97.50 & 1 \\
    \midrule
    \multirow{3}{*}{\shortstack[l]{DT}} 
      & Spikeformer & 110 & 81.55 & 42.02 & 79.38 & 91.80 & 4 \\
      & SpikeBERT & 109 & 85.39 & 46.11 & 80.69 & 93.00 & 4 \\
      & SpikeGPT & 45 & 80.39 & 37.69 & 69.23 & 88.45 & 50 \\
    \midrule
    \multirow{4}{*}{A2S} 
      & \multirow{2}{*}{\centering SpikeZIP-TF} 
        & 125 & 92.81 & 52.71 & 86.13 & 95.55 & 64 \\
      & & 355 & 93.79 & 56.51 & 89.28 & 96.70 & 128 \\
      \cmidrule{2-8}
      & \multirow{2}{*}{\centering \makecell[l]{Ours}} 
        & 125 & 93.46 & 55.11 & 89.00 & 96.30 & 16 \\
      & & 355 & \textbf{95.98} & \textbf{58.31} & \textbf{90.96} & \textbf{97.45} & 16 \\
    \bottomrule
  \end{tabular}
  \caption{NLU Performance Comparison}
    \label{tab:nlu}
\end{table}

\begin{figure*}[htbp]
  \centering
  \includegraphics[height=5cm, width=\linewidth]{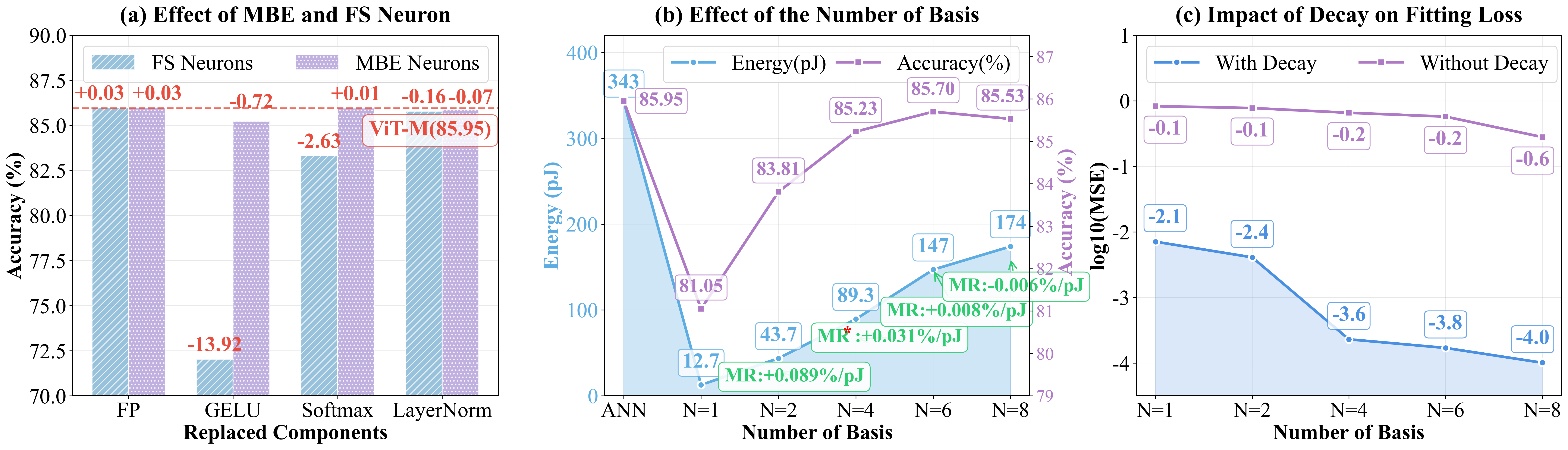}
  \caption{(a) Effect of MBE and FS neuron on components in Transformer,  FP in (a) means floating multiplication. (b) Impact of MBE neurons with varying basis counts on energy consumption and accuracy when replacing GELU. MR represents marginal rate of change between accuracy and energy consumption. (c) Effect of the basis and decay. Loss comparison between models with (blue) and without (red) decay mechanism across basis numbers N=1,2,4,6,8. Y-axis shows log MSE loss (lower is better).
}
  \label{fig:Basis_N}
\end{figure*}

\paragraph{\textbf{Comparison on NLG}}
To demonstrate the applicability of our method to NLG tasks, we evaluate its performance on WikiText-2 and WikiText-103. As shown in Tab.~\ref{tab:nlg}, our method achieves 22.69 perplexity on WikiText-2 with only T=16 timesteps, incurring merely 0.35\% conversion loss compared to GPT-2. More significantly, on WikiText-103, it attains 23.41 perplexity—a substantial 41.1\% improvement over directly-trained SpikeGPT (39.75) while using dramatically fewer timesteps (T=16 vs. T=1024), closely approaching the original GPT-2 performance and demonstrating superior efficiency for long-context generation.

\begin{table}[h]
  \setlength{\tabcolsep}{1.6pt} 
  \begin{tabular}{llcccc}
    \toprule
    Category & Model & Param & Wiki-2 ({↓})  & Wiki-103 ({↓})  & T \\
    \midrule
    \multirow{1}{*}{ANN} 
      & GPT-2 & 346 & 22.34 & 22.65 & 1 \\
    \midrule
    \multirow{1}{*}{DT} 
      & SpikeGPT &216  & \textbf{18.01} & 39.75 & 1024 \\
    \midrule
    A2S & \makecell[l]{Ours}  & 346 &\makecell[c]{22.69 \\ {(+0.35)} } &\makecell[c]{\textbf{23.41} \\ {(+0.76)}} & 16 \\
    \bottomrule
  \end{tabular}
\caption{NLG Performance Comparison. Both Wiki-2 and Wiki-103 use perplexity, {↓} indicates that lower is better.}
  \label{tab:nlg}
\end{table}

\subsection{Ablation Study}

To compare MBE neurons and FS neurons, we evaluate their performance on key Transformer components. For GELU and FP multiplication, direct replacement is employed. As shown in Fig.~\ref{fig:Basis_N}(a), MBE neurons surpass FS neurons by $13.20\%$ in GELU conversion and achieve improvements in Softmax and LayerNorm, demonstrating superior approximation ability. In Fig.~\ref{fig:Basis_N}(b), increasing $N$ leads to more accurate function approximations by MBE neurons. MBE neurons achieve high performance (e.g., less than 1\% conversion loss when N $\geq 4$) while offering significantly higher energy efficiency than ANNs. N=4 achieves the optimal accuracy-MR trade-off, representing the knee point where marginal gains plateau.
The synergistic effect between exponential decay and multi-basis encoding is demonstrated in Fig.~\ref{fig:Basis_N}(c). Incorporating decay consistently reduces MSE by $1$--$2$ orders of magnitude across all basis numbers, with significant performance improvements. This reveals that exponential decay creates a multiplicative synergy with multi-basis encoding, enabling exceptional approximation precision unattainable through basis expansion alone.

\begin{table}[h]
\centering
\setlength{\tabcolsep}{2.0pt} 
\begin{tabular}{l l c c c c c}
\toprule
\multirow{2}{*}{\centering Dataset} & 
\multirow{2}{*}{\centering Model} & 
\multirow{2}{*}{\centering ANN} & 
\multicolumn{4}{c}{T} \\
\cmidrule(lr){4-7}
 & & & 8 & 10 & 12 & 16 \\
\midrule
\multirow{2}{*}{ImageNet} 
    & ViT-B/16 & 83.44 & 0.12 & 79.96 & 82.79 & 83.00 \\
    & ViT-M/16 & 85.95 & 1.17 & 21.99 & 84.43 & 85.31 \\
\midrule
\multirow{2}{*}{MR} 
    & Roberta-B & 89.39 & 50.17  & 71.30  & 88.55  & 89.00 \\
    & Roberta-L & 91.36 & 49.66 & 68.62 & 90.52 & 90.96 \\
\midrule
Wiki-103 ({↓}) 
    & GPT-2 & 22.65 & 41072 & 11992 & 33.56 & 23.41 \\
\bottomrule
\end{tabular}
\caption{Performance across various timesteps and models. Wiki-103 uses perplexity, {↓} indicates that lower is better.}
\label{tab:ffsc_ablation}
\end{table}

We evaluate ImageNet, MR, and Wiki-103 to analyze timestep requirements for A2S conversion across vision and language tasks. As shown in Tab.~\ref{tab:ffsc_ablation}, our method achieves about 1\% conversion loss for ViT, RoBERTa, and GPT-2 at T=16, with near-optimal performance at T=12, demonstrating efficient progressive encoding. Notably, the accuracy of ViT-M/16 drops to 21.99\% at T=10, primarily due to its wider identity mapping range [0,62] causing amplified approximation errors under limited timesteps.

\section{Energy Estimation}
Unlike ANNs where energy depends on floating-point operations (FLOPs), the energy cost of SNNs is dominated by synaptic operations (SOPs). Following~\cite{rathi2020diet}, the energy ratio between SNNs and ANNs  is:
\begin{equation}
\label{eq:nenergy_comsumption}
\frac{E_{\text{SNN}}}{E_{\text{ANN}}} = \frac{\text{SOPs} \cdot E_{\text{AC}}}{\text{FLOPs} \cdot E_{\text{MAC}}},
\end{equation}
where $E_{\text{MAC}}$ = 4.6pJ, $E_{\text{AC}}$ = 0.9pJ. For ViT-M/16, we measure the firing rates $\eta$ of nonlinear operations across all layers. Taking GELU as an example, the ANN implementation requires 70 FLOPs~\cite{jiang2024spatio}, while the MBE neuron achieves $\eta$=38.22\% under N=4 and T=16. 
The synaptic operations achieve energy consumption of 13.7\% according to $T*N*\eta*E_{AC}$.
Other operations yield greater energy savings owing to lower MBE firing rates. Complete results and detailed firing rates are provided in Appendix G.4.

\section{Conclusion}

This paper proposes an efficient Spiking Transformers conversion method that requires no additional training on the source ANNs. 
We first theoretically and experimentally identify two key challenges in using FS neurons to approximate nonlinear operations: excessive dependence on initialization (EDI) and global suboptimality (GSO) problems. To address these issues, the MBE neuron employs an exponential decay strategy and a multi-basis encoding method to more accurately approximate the nonlinear operations in Transformer architecture. Based on the MBE neuron, a general ANN-to-SNN conversion framework is developed, which supports spiking nonlinear activation functions, spiking FP multiplications, spiking Softmax, and spiking LayerNorm. Extensive experiments on various models (CNN, ViT, RoBERTa, GPT-2) across tasks (CV, NLU, NLG) confirm that the proposed method achieves near-lossless conversion. Therefore, our method provides a promising approach for the efficient and scalable deployment of Transformer-based SNNs in real-world applications.

\section{Acknowledgments}
This work is supported in part by the National Natural Science Foundation of China (62220106008 and 62576080), in part by the State Key Laboratory of Brain Cognition and Brain-inspired Intelligence Technology, Grant No. SKLBI-K2025010.

\bibliography{aaai2026}

\end{document}